\documentclass[10pt]{article}
\usepackage[margin=1.2in]{geometry}
\usepackage{url}
\usepackage{amsmath}
\usepackage{amssymb}
\usepackage{amsthm}
\usepackage{algorithm2e}
\usepackage{natbib}
\usepackage{graphicx}
\usepackage{subcaption}
\usepackage{afterpage}
\usepackage{authblk}

\newcommand{\mx}{\mathcal{X}}

\newcommand{\md}{\mathcal{D}}
\newcommand{\mn}{\mathcal{N}}
\newcommand{\ms}{\mathcal{S}}
\newcommand{\mw}{\mathcal{W}}

\newcommand{\pti}{p_t^{(i)}}
\newcommand{\qti}{q_t^{(i)}}

\newcommand{\naturals}{\mathbb{N}}
\newcommand{\reals}{\mathbb{R}}
\newcommand{\inner}[1]{\left\langle #1 \right\rangle}
\newcommand{\norm}[1]{\| #1 \|}

\newcommand{\wvec}[1]{\ww^{(#1)}}
\newcommand{\uvec}[1]{\uu^{(#1)}}

\newcommand{\net}{N_{\mw}}
\newcommand{\loss}[1]{L_{#1} \left( W\right)}

\newcommand{\elm}{\boldsymbol{e}}
\DeclareMathOperator{\sign}{sign}
\newtheorem{thm}{Theorem}[section]
\newtheorem{cor}[thm]{Corollary}
\newtheorem{prop}[thm]{Proposition}
\newtheorem{lem}[thm]{Lemma}
\usepackage{boxedminipage}

\renewcommand{\xi}{{\xx}^{(m)}}

\newcommand{\prob}{\mathbb{P}}

\newcommand{\cB}{\mathcal{B}}
\newcommand{\needcite}[1]{}

\newcommand{\be}{\begin{equation}}
\newcommand{\ee}{\end{equation}}
\newcommand{\benn}{\begin{equation*}}
\newcommand{\eenn}{\end{equation*}}
\newcommand{\bea}{\begin{eqnarray*}}
\newcommand{\eea}{\end{eqnarray*}}
\newcommand{\bean}{\begin{eqnarray}}
\newcommand{\eean}{\end{eqnarray}}

\newcommand{\ww}{\boldsymbol{w}} 
\newcommand{\xx}{\boldsymbol{x}}

\newcommand{\yy}{\boldsymbol{y}}

\newcommand{\vv}{\boldsymbol{v}}

\newcommand{\uu}{\boldsymbol{u}}

\newcommand{\agtodo}[1]{({\textcolor{red}{#1}})}

\newcommand{\ignore}[1]{}

\newcommand{\polyring}[1]{\reals\left[x_1,\ldots,x_n\right]}

{
\begin{center}
\begin{boxedminipage}{0.8\linewidth}
\begin{center}
\textbf{\texttt{#1}}
\end{center}
\rm
\begin{tabbing}
....\=...\=...\=...\=...\=  \+ \kill
} %
{\end{tabbing} 
\end{boxedminipage} \end{center} %\vspace{0.3cm}
}

\renewcommand{\eqref}[1]{Eq.~\ref{#1}}

\newcommand{\secref}[1]{Section \ref{#1}}

\newcommand{\expect}[2]{\mathbb{E}_{#1}\left[{#2}\right]}

\newtheorem{theorem}{Theorem}
\newtheorem{lemma}[theorem]{Lemma}

\newtheorem{remark}[theorem]{Remark}

\title{SGD Learns Over-parameterized Networks that Provably Generalize on Linearly Separable Data}
\author[1]{Alon Brutzkus}
\author[1]{Amir Globerson}
\author[2]{Eran Malach}
\author[2]{Shai Shalev-Shwartz}
\affil[1]{The Blavatnik School of Computer Science, Tel Aviv University}
\affil[2]{School of Computer Science and Engineering, The Hebrew University}

\date{\today}

	\date{\vspace{-5ex}}

\begin{document}
\maketitle

\begin{abstract} 
	Neural networks exhibit good generalization behavior in the
	\textit{over-parameterized} regime, where the number of network parameters
	exceeds the number of observations. Nonetheless,
	current generalization bounds for neural networks fail to explain this
	phenomenon. In an attempt to bridge this gap, we study the problem of
	learning a two-layer over-parameterized neural network, when the data is generated by a linearly separable function. In the case where the network has Leaky
	ReLU activations, we provide both optimization and generalization guarantees for over-parameterized networks.
	Specifically, we prove convergence rates of SGD to a global
	minimum and provide generalization guarantees for this global minimum
	that are independent of the network size. 
	%This is a surprising result since, as we show, there exist multiple global minima that overfit the training set. 
	Therefore, our result clearly shows that the use of SGD for optimization both finds a global minimum, and avoids overfitting despite the high capacity of the model. This is the first theoretical demonstration that SGD can avoid overfitting, when learning over-specified neural network classifiers.
\end{abstract}
%\tableofcontents
%\pagebreak

\section{Introduction}
\label{sec:introduction}

Neural networks have achieved remarkable performance in many machine learning tasks. Although recently there have been numerous theoretical contributions to understand their success, it is still largely unexplained and remains a mystery. In particular, it is not known why in the over-parameterized setting, in which there are far more parameters than training points, stochastic gradient descent (SGD) can learn networks that generalize well, as been observed in practice \citep{neyshabur2014search,zhang2016understanding}.

In such over-parameterized settings, the loss function can contain multiple global minima that generalize poorly. Therefore, learning can in principle lead to models with low training
error, but high test error. However, as often observed in practice, SGD is in fact able to find models with low training error {\em and} good generalization performance. This suggests that the optimization procedure, which depends on the optimization method (SGD) and the training data, introduces some form of \textit{inductive bias} which directs it towards a low complexity solution. Thus, in order to explain the success of neural networks, it is crucial to characterize this inductive bias and understand what are the guarantees for generalization of over-parameterized neural networks.

In this work, we address these problems in a binary classification setting where SGD optimizes a two-layer over-parameterized network with the goal of learning a \textit{linearly separable} function. Clearly, an over-parameterized network is not necessary for classifying linearly separable data, since this is possible with linear classifiers (e.g., with the Perceptron algorithm) which also have good generalization guarantees \citep{shalev2014understanding}. But, the key question which we address here is whether a large network will overfit
in such a case or not. As we shall see, it turns out that although the networks we consider are rich enough to considerably overfit the data, this does not happen when SGD is 
used for optimization. In other words, SGD introduces an inductive bias which allows it to learn over-parameterized networks that can generalize well. Therefore, this setting serves as a good test bed for studying the effect of over-paramaterization.

\section{Problem Formulation}
\label{sec:formulation}
Define $\mathcal{X} = \{ \xx \in \mathbb{R}^d : \|\xx\| \le 1\}$,
$\mathcal{Y}=\left\{ \pm1 \right\} $. We consider a distribution over linearly separable points. Formally, let $\md$ be a distribution over $\mathcal{X} \times \mathcal{Y}$ such that there exists $\ww^* \in \reals^d$ for which $\prob_{(\xx,y)\sim\md}(y \inner{\ww^*, \xx} \ge 1) = 1$. \footnote{This implies that $\norm{\ww^*} \geq 1$.} Let $S = \{(\xx_1,y_1), \dots, (\xx_n,y_n) \} \subseteq \mathcal{X} \times \mathcal{Y}$ be a training set sampled i.i.d. from $\md$. \footnote{Without loss of generality, we will ignore the event that $y_i \inner{\ww^*, \xx_i} < 1$ for some $i$, since this is an event of measure zero. }

Consider the following two-layer neural network, with $2k>0$ hidden units. \footnote{We have an even number of hidden neurons for ease of exposition. See the definition of $\vv$ below. } The network parameters are $W \in \reals^{2k \times d}, \vv \in \reals^{2k}$, which we denote jointly by $\mw=(W,\vv)$. The network output is given by the function $\net{}:\reals^d \rightarrow \reals$ defined as:
\begin{equation}
\label{eq:network}
\net(\xx) =  \vv^\top\sigma(W\xx)
\end{equation}
where $\sigma$ is a non-linear activation function applied element-wise. \ignore{For simplicity we consider vectors $\vv$ defined via a single scalar $v\in\reals^+$ as follows: $\vv = (\overbrace{v \dots v}^k, \overbrace{-v \dots -v}^k)$ such that $v > 0$. In our setting we fix the }

We define the empirical loss over $S$ to be the mean hinge-loss:
\[
\loss{S} = \frac{1}{n}\sum_{i=1}^{n}{\max{\{1-y_i\net(\xx_i), 0\}}}
\]
Note that for convenience of analysis, we will sometimes refer to $L_S$ as a function over a vector. Namely, for a matrix $W \in \reals^{2k \times d}$, we will consider instead its vectorized version $\vec{W} \in \reals^{2kd}$ (where the rows of $W$ are concatenated) and define, with abuse of notation, that $L_S(\vec{W}) = L_S(W)$.

In our setting we fix the second layer to be $\vv = (\overbrace{v \dots v}^k, \overbrace{-v \dots -v}^k)$ such that $v > 0$ and only learn the weight matrix $W$. We will consider only positive homogeneous activations (Leaky ReLU and ReLU) and thus the network we consider with $2k$ hidden neurons is as expressive as networks with $k$ hidden neurons and any vector $\vv$ in the second layer. \footnote{For example, consider a network with $k$ hidden neurons with positive homogeneous activations, where each hidden neuron $i$ has incoming weight vector $\ww_i$ and outgoing weight $v_i$. Then we can express this network with the network defined in \eqref{eq:network} as follows. For each $i$ such that $v_i > 0$, we define a neuron in the new network with incoming weight vector $\frac{\ww_i}{v_i}$ and outgoing weight $1$. Similarly, if $v_i < 0$, we define a neuron in the new network with incoming weight vector $\frac{\uu_i}{-v_i}$ and outgoing weight $-1$. For all other neurons we define an incoming zero weight vector. Due to the positive homogeneity, it follows that this network is equivalent to the network with $k$ hidden neurons.} Hence, we can fix the second layer without limiting the expressive power of the two-layer network. Although it is relatively simpler than the case where the second layer is not fixed, the effect of over-parameterization can be studied in this setting as well.

Hence, the objective of the optimization problem is to find:
\begin{equation}
\label{eq:optimization}
\arg\min_{W \in \reals^{2k \times d}} \loss{S}
\end{equation}
where $\min\limits_{W \in \reals^{2k \times d}} \loss{S} =0$ holds for the activations we will consider (Leaky ReLU and ReLU).

We focus on the case where $\loss{S}$ is minimized using an SGD algorithm with batch of size $1$, and where only the weights of the first layer (namely $W$) are updated. At iteration $t$, SGD randomly chooses a point $(\xx_t,y_t) \in S$ and updates the weights with a constant learning rate $\eta$. Formally, let $\mw_t = (W_t,\vv)$ be the parameters at iteration $t$, then the update at iteration $t$ is given by
\begin{equation}
\label{eq:sgd_update}
W_t = W_{t-1} - \eta \frac{\partial}{\partial W}L_{\{(\xx_t,y_t)\}}(W_{t-1}) 
\end{equation}

 We define a \textit{non-zero} update at iteration $t$ if it holds that $\frac{\partial}{\partial W}L_{\{(\xx_t,y_t)\}}(W_{t-1}) \neq 0$. Finally, we will need the following notation. For $1 \leq i \leq k$, we denote by $\wvec{i}_t \in \reals^d$ the incoming weight vector of neuron $i$ at iteration $t$. \footnote{These are the neurons with positive outgoing weight $v > 0$.}
Similarly, for $1 \leq i \leq k$ we define $\uvec{i}_t \in \reals^d$ to be the incoming weight vector of neuron $k + i$ at iteration $t$.

\section{Main Result}
We now present our main results, for the case where $\sigma$ is the Leaky ReLU function. Namely, $\sigma(z)=\max\{\alpha z,z\}$ where $0 < \alpha < 1$.

First, we show that SGD can find a global optimum of $\loss{S}$. Note that this is by no means obvious, since $\loss{S}$ is a non-convex function (see Proposition \ref{prop:loss_properties}). Specifically, we show that SGD converges to such an optimum while making at most:
\begin{equation}
M = \frac{\norm{\ww^*}^2}{\alpha^2} + O\Bigg(\frac{\norm{\ww^*}^2}{\min\{\eta,\sqrt{\eta}\}}\Bigg)
\end{equation}
\textit{non-zero} update steps (see Corollary \ref{cor:convergence}). In particular, the bound is \textit{independent} of the number of neurons $2k$. To the best of our knowledge, this is the first convergence guarantee of SGD for neural networks with the hinge loss. Furthermore, we prove a lower bound of $ \Omega\Big(\frac{\norm{\ww^*}}{\eta} + \norm{\ww^*}^2\Big)$ for the number of non-zero updates (see Theorem \ref{thm:lower_bound}). 

Next, we address the question of generalization. As noted earlier, since the network is large, it can in principle overfit. Indeed, there are parameter settings for which the network will have arbitrarily bad test error (see Section \ref{sec:expressiveness}). However, as we show here, this will not happen in our setting where SGD is used for optimization. In Theorem \ref{thm:generalization_global_min} we use a compression bound to show that the model learned by SGD will have a generalization error of $O\Big(\frac{M\log{n}}{n}\Big)$.\footnote{See discussion in Remark \ref{rem:eta_dependence} on the dependence of the generalizaion bound on $\eta$.} This implies that for {\em any} network size, given a sufficiently large number of training samples that is \textit{independent} of the network size, SGD converges to a global minimum with good generalization behaviour. This is despite the fact that for sufficiently large $k$ there are multiple global minima which overfit the training set (see Section \ref{sec:expressiveness}). This implies that SGD is biased towards solutions that can be expressed by a small set of training points and thus generalizes well.

%Consequently, by applying a compression bound we obtain a generalization guarantee for over-parameterized networks. Formally, we get a generalization guarantee of $O\Big(\frac{M\log{n}}{n}\Big)$ for the global minimum obtained by SGD (see Theorem \ref{thm:generalization_global_min}). \footnote{See discussion in Remark \ref{rem:eta_dependence} on the dependence of the generalizaion bound on $\eta$.} 

\ignore{
Finally, we characterize the inductive bias introduced by SGD and the linearly separable data. Namely, that the cosine similarity measure between the vector of all network parameters and the $2kd$-dimensional vector $(\overbrace{\ww^* \dots \ww^*}^k, \overbrace{-\ww^* \dots -\ww^*}^k)$, is lower bounded by a function which increases in each iteration and approaches 1 as SGD makes more updates. This implies that for any network size, SGD is biased towards a solution which is the true linear classifier
}

To summarize, when the activation is the Leaky ReLU and the data is linearly separable, we provide provable guarantees of optimization, generalization and expressive power for over-parameterized networks. This allows us to provide a rigorous explanation of the performance of over-parameterized networks in this setting. This is a first step in unraveling the mystery of the success of over-parameterized networks in practice.

We further study the same over-parameterized setting where the non-linear activation is the ReLU function (i.e., $\sigma(z)=\max\{0,z\}$). Surprisingly, this case has different properties. Indeed, we show that the loss contains spurious local minima and thus the previous convergence result of SGD to a global minimum does not hold in this case. Furthermore, we show an example where over-parameterization is favorable from an optimization point of view. Namely, for a sufficiently small number of hidden neurons, SGD will converge to a \textit{local} minimum with high probability, whereas for a sufficiently large number of hidden neurons, SGD will converge to a \textit{global} minimum with high probability.

The paper is organized as follows. We discuss related work in Section \ref{sec:related_work} . In Section \ref{sec:convergence} we prove the convergence bounds, in Section \ref{sec:generalization} we give the generalization guarantees and in Section \ref{sec:relu} the results for the ReLU activation. We conclude our work in Section \ref{sec:discussion}.

%Complexity measures such as VC dimension, Rademacher complexity and uniform stability have been shown to be inadequate for explaining this phenomenon \citep{ZhangBHRV16}. 

\section{Related Work}
\label{sec:related_work}

The generalization performance of neural networks has been studied extensively. Earlier results \citep{anthony2009neural} provided bounds that depend on the VC dimension of the network, and the VC dimension was shown to scale linearly with the number of parameters. More recent works, study alternative notions of complexity, such as Rademacher compexity \citep{bartlett2002rademacher,neyshabur2015norm,bartlett2017spectrally, kawaguchi2017generalization}, Robustness \citep{xu2012robustness} and PAC-Bayes \citep{neyshabur2017pac}. However, all of these notions fail to explain the generalization performance of \textit{over-parameterized} networks \citep{neyshabur2017exploring}. This is because these bounds either depend on the number of parameters or on the number of hidden neurons (directly or indirectly via norms of the weights) and become loose when these quantities become sufficiently large. The main disadvantage of these approaches, is that they do not depend on the optimization method (e.g., SGD), and thus do not capture its role in the generalization performance. In our work, we give generalization guarantees based on a compression bound that follows from convergence rate guarantees of SGD, and thus take into account the effect of the optimization method on the generalization performance. This analysis results in generalization bounds that are independent of the network size and thus hold for over-parameterized networks. 

In parallel to our work, \cite{kawaguchi2017generalization} give generalization bounds for neural networks that are based on Rademacher complexity. Here too, the analysis does not take into account the optimization algorithm and the bound depends on the norm of the weights. Therefore, the bound can become vacuous for over-parameterized networks.

Stability bounds for SGD in non-convex settings were given in \citet{hardt2016gradient, kuzborskij2017data}. However, their results hold for smooth loss functions, whereas the loss function we consider is not smooth due to the non-smooth activation functions (Leaky ReLU, ReLU). 

Other works have studied generalization of neural networks in a model recovery setting, where assumptions are made on the underlying model and the input distribution \citep{brutzkus2017globally,zhong2017recovery,li2017convergence,du2017convolutional,tian2017analytical}. However, in their works the neural networks are not over-parameterized as in our setting.

\citet{soltanolkotabi2017theoretical} analyze the optimization landscape of over-parameterized networks and give convergence guarantees for gradient descent to a global minimum when the data follows a Gaussian distribution and the activation functions are differentiable. The main difference from our work is that they do not provide generalization guarantees for the resulting model. Furthermore, we do not make any assumptions on the distribution of the feature vectors.

In a recent work, \citet{nguyen2017loss} show that if training points are linearly separable then under assumptions on the rank of the weight matrices of a fully-connected neural network, every critical point of the loss function is a global minimum. Their work extends previous results in \citet{gori1992problem,frasconi1997successes,yu1995local}. Our work differs from these in several respects. First, we show global convergence guarantees of SGD, whereas they only analyze the optimization landscape, without direct implications on performance of optimization methods. Second, we provide generalization bounds and their focus is solely on optimization. Third, we consider non-differentiable activation functions (Leaky ReLU, ReLU) while their results hold only for continuously differentiable activation functions.

\section{Convergence Analysis}
\label{sec:convergence}
In this section we consider the setting of \secref{sec:formulation} with a leaky ReLU activation function. In \secref{sec:upper_bound} we show SGD will converge to a globally optimal solution, and analyze the rate of convergence. In \secref{sec:upper_bound} we also provide lower bounds on the rate of convergence. The results in this section are interesting for two reasons. First, they show convergence of SGD for a non-convex objective. Second, the rate of convergence results will be used to derive generalization bounds in \secref{sec:generalization}.

\subsection{Upper Bound}
\label{sec:upper_bound}

Before proving convergence of SGD to a global minimum, we show that every critical point is a global minimum and the loss function is non-convex. The proof is deferred to the appendix.

\begin{prop}
	\label{prop:loss_properties}
	$\loss{S}$ satisfies the following properties: 1) Every critical point is a global minimum. 2) It is non-convex.
	\ignore{
		\begin{enumerate}
			\item Every critical point is a global minimum.
			\item It is non-convex.
		\end{enumerate}
	}
\end{prop} 

Let $\vec{W}_t = (\wvec{1}_t \dots \wvec{k}_t \uvec{1}_t \dots \uvec{k}_t ) \in \reals^{2kd}$ be the vectorized version of $W_t$ and $N_t := N_{\mw_t}$ where $\mw_t=(W_t,\vv)$ (see \eqref{eq:network}). Since we will show an upper bound on the number of non-zero updates, we will assume for simplicity that for all $t$ we have a non-zero update at iteration $t$.

We assume that SGD is initialized such that the norms of all rows of $W_0$ are upper bounded by some 
constant $R>0$. Namely for all $1\leq i \leq k$ it holds that: 
\begin{equation}
\label{eq:initialization}
\norm{\wvec{i}_0}, \norm{\uvec{i}_0} \le R
\end{equation}
Define $M_k := \frac{\norm{\ww^*}^2}{\alpha^2}   + \frac{\norm{\ww^*}^2}{k\eta v^2\alpha^2} + \frac{\sqrt{R(8k^2 \eta^2 v^2+ 8\eta k)}{\norm{\ww^*}}^{1.5}}{2k (\eta v\alpha)^{1.5}} +\frac{2R\norm{\ww^*}}{\eta v\alpha}$. We  give an upper bound on the number of non-zero updates SGD makes until convergence to a critical point (which is a global minimum by Proposition \ref{prop:loss_properties}). The result is summarized in the following theorem. 
\begin{theorem}\label{thm:convergence}
	SGD converges to a global minimum after performing at most 
	$M_k$ non-zero updates.
\end{theorem}

We will briefly sketch the proof of Theorem \ref{thm:convergence}. The full proof is deferred to the Appendix (see Section \ref{sec:thm_convergence_proof}). The analysis is reminiscent of the Perceptron convergence proof (e.g. in \citet{shalev2014understanding}), but with key modifications due to the non-linear architecture. Concretely, assume SGD performed $t$ non-zero updates. We consider the vector $\vec{W_t}$ and the vector $\vec{W}^* = (\overbrace{\ww^* \dots \ww^*}^k, \overbrace{-\ww^* \dots -\ww^*}^k) \in \reals^{2kd}$ which is a global minimum of $L_S$. We define $F(W_t) = \inner{\vec{W}_t, \vec{W}^*}$ and  $G(W_t) = \norm{\vec{W}_t}$. Then, we give an upper bound on $G(W_t)$ in terms of $G(W_{t-1})$ and by a recursive application of inequalities we show that $G(W_t)$ is bounded from above by a square root of a linear function of $t$. Similarly, by a recursive application of inequalities, we show that $F(W_t)$ is bounded from 
below by a linear function of $t$. Finally, we use the Cauchy-Schwartz inequality $\frac{\left|F(W_t)\right|}{G(W_t)\norm{\vec{W}^*}} \leq 1$ to show that $t \leq M_k$.

To obtain a simpler bound than the one obtained in Theorem \ref{thm:convergence}, we use the fact that we can set $R,v$ arbitrarily, and choose:\footnote{This initialization resembles other initializations that are used in practice \citep{bengio2012practical,glorot2010understanding}} 
\begin{equation}
\label{eq:nn_init}
R =v=\frac{1}{\sqrt{2k}} \text{.} 
\end{equation}
Then by Theorem \ref{thm:convergence} we get the following. The derivation is given in the Appendix (Section \ref{sec:cor_convergence_proof}).
\begin{cor}
	\label{cor:convergence}
	Let $R =v=\frac{1}{\sqrt{2k}}$, then SGD converges to a global minimum after perfoming at most $M_k = \frac{\norm{\ww^*}^2}{\alpha^2}  
	+ O\Bigg( \frac{\norm{\ww^*}^2}{\min\{\eta,\sqrt{\eta}\}}\Bigg)$ non-zero updates.
\end{cor}

Thus the bound consists of two terms, the first which only depends on the margin (via $\norm{\ww^*}$) and the second which scales inversely with $\eta$. More importantly, the bound is independent of the network size. 

\subsection{Lower Bound}
\label{sec:lower_bound}

We use the same notations as in Section \ref{sec:upper_bound}. The lower bound is given in the following theorem, which is proved in the Appendix (Section \ref{sec:proof_lower_bound}). 

\begin{theorem}
	\label{thm:lower_bound}
	Assume SGD is initialized according to \eqref{eq:nn_init}, then for any $d$ there exists a sequence of linearly separable points on which SGD will make at least $\Omega\Big(\frac{\norm{\ww^*}}{\eta} + \norm{\ww^*}^2\Big)$ mistakes.
\end{theorem}

Although this lower bound is not tight, it does show that the upper bound in Corollary \ref{cor:convergence} cannot be much improved. Furthermore, the example presented in the proof of Theorem \ref{thm:lower_bound}, demonstrates that $\eta \rightarrow \infty$ can be optimal in terms of optimization and generalization, i.e., SGD makes the minimum number of updates ($\norm{\ww^*}^2$) and the learned model is equivalent to the true classifier $\ww^*$. We will use this observation in the discussion on the dependence of the generalization bound in Theorem \ref{thm:generalization_global_min} on $\eta$ (see Remark \ref{rem:eta_dependence}). 

\section{Generalization}
\label{sec:generalization}

In this section we give generalization guarantees for SGD learning of over-parameterized networks with Leaky ReLU activations. These results are obtained by combining  Theorem \ref{thm:convergence} with a compression generalization bound  (see Section \ref{sec:compression}). In Section \ref{sec:expressiveness} we show that over-parameterized networks are sufficiently expressive to contain global minima that overfit the training set. Taken together, these results show that although there are models that overfit, SGD effectively avoids these, and finds the models that generalize well.  

\subsection{Compression Bound}
\label{sec:compression}

Given the bound in Theorem \ref{thm:convergence} we can invoke compression bounds for generalization guarantees with respect to the $0$-$1$ loss \citep{littlestone1986relating} . Denote by $N_k$ a two-layer neural network with $2k$ hidden neurons defined in Section \ref{sec:introduction} where $\sigma$ is the Leaky ReLU. Let $SGD_k(S,W_0)$ be the output of running SGD for training this network on a set $S$ and initialized with $W_0$ that satisfies \eqref{eq:initialization}. Define $\mathcal{H}_k$ to be the set of all possible hypotheses that $SGD_k(S,W_0)$ can output for \textit{any} $S$ and $W_0$ which satisfies \eqref{eq:initialization}.

Now, fix an initialization $W_0$. Then the key observation is that by Theorem $\ref{thm:convergence}$ we have $SGD_k(S,W_0)=B_{W_0}(\xx_{i_1},...,\xx_{i_{c_k}})$  for $c_k \leq M_k$, some function $B_{W_0}:\mathcal{X}^{c_k}\rightarrow \mathcal{H}_k$ and $(i_1,...,i_{c_k}) \in [n]^{c_k}$.\footnote{We use a subscript $W_0$ because the function is determined by $W_0$.} Equivalently, $SGD_k(\cdot,W_0)$ and $B_{W_0}$ define a compression scheme of size $c_k$ for hypothesis class $\mathcal{H}_k$ (see Definition 30.4 in \cite{shalev2014understanding}). Denote by $V=\{\xx_j : j \notin \{i_1,...,i_{c_k}\}\}$ the set of examples which were not selected to define $SGD_k(S,W_0)$. Let $L_{\mathcal{D}}^{0-1}(SGD_k(S,W_0))$ and $L_V^{0-1}(SGD_k(S,W_0))$ be the true risk of $SGD_k(S,W_0)$ and empirical risk of $SGD_k(S,W_0)$ on the set $V$, respectively. Then by Theorem 30.2 and Corollary 30.3 in \cite{shalev2014understanding} we can easily derive the following theorem. The proof is deferred to the Appendix (Section \ref{sec:proof_thm_generalization}).
\ignore{
\begin{theorem}
	\label{thm:generalization}
	Let $n \geq 2c_k$, then with probability of at least $1-\delta$ over the choice of $S$ we have 
	$$L_{\mathcal{D}}(SGD_k(S,W_0)) \leq L_{V}(SGD_k(S,W_0)) + \sqrt{L_{V}(SGD_k(S,W_0)) \frac{4c_k\log{\frac{n}{\delta}}}{n}} + \frac{8c_k\log{\frac{n}{\delta}}}{n}$$
	
	Furthermore, assuming $L_{V}(SGD_k(S,W_0)) = 0$ we have:\agtodo{We can skip this part in this corollary and the next, and just use it directly in Theorem 4..}
	
	$$L_{\mathcal{D}}(SGD_k(S,W_0)) \leq  \frac{8c_k\log{\frac{n}{\delta}}}{n}$$
\end{theorem} 

Theorem \ref{thm:generalization} holds for a \textit{fixed} initialization $W_0$. The following corollary shows that the same result holds with high probability over $S$ \textit{and} $W_0$, where $W_0$ is chosen independently of $S$ and satisfies \eqref{eq:initialization}. The proof is deferred to the Appendix (Section \ref{sec:proof_cor_generalization}).
}
\begin{theorem}
	\label{thm:leaky_generalization}
	Let $n \geq 2c_k$, then with probability of at least $1-\delta$ over the choice of $S$ and $W_0$ we have 
	$$L_{\mathcal{D}}^{0-1}(SGD_k(S,W_0)) \leq L_{V}^{0-1}(SGD_k(S,W_0)) + \sqrt{L_{V}^{0-1}(SGD_k(S,W_0)) \frac{4c_k\log{\frac{n}{\delta}}}{n}} + \frac{8c_k\log{\frac{n}{\delta}}}{n}$$
\end{theorem}

Since $L_{V}^{0-1}(SGD_k(S,W_0)) = 0$ holds at a global minimum of $L_S$, then by Combining the results of Corollary \ref{cor:convergence} and Theorem \ref{thm:leaky_generalization}, we get the following theorem. 

\begin{theorem}
\label{thm:generalization_global_min}
If $n \geq 2c_k$ and assuming the initialization defined in \eqref{eq:nn_init}, then with probability at least $1-\delta$ over the choice of $S$ and $W_0$, SGD converges to a global minimum of $L_S$ with $0$-$1$ test error at most 
\begin{equation}
\label{eq:generalization_bound}
\frac{8}{n}\Bigg(\frac{\norm{\ww^*}^2}{\alpha^2} + O\Bigg(\frac{\norm{\ww^*}^2}{\min\{\eta,\sqrt{\eta}\}}\Bigg)\Bigg)\log{\frac{n}{\delta}}
\end{equation}	
\end{theorem}

 Thus for fixed $\norm{\ww^*}$ and $\eta$ we obtain a sample complexity guarantee that is independent of the network size (See Remark \ref{rem:eta_dependence} for a discussion on the dependence of the bound on $\eta$). This is despite the fact that for sufficiently large $k$, the network has \textit{global} minima that have arbitrarily high test errors, as we show in the next section. 
Thus, SGD and the linearly separable data introduce an inductive bias which directs SGD to the global minimum with low test error while avoiding global minima with high test error. In Figure \ref{fig:overparam} we demonstrate this empirically for a linearly separable data set (from a subset of MNIST) learned using over-parameterized networks. The figure indeed shows that SGD converges to a global minimum which generalizes well.

\begin{figure}	
	\begin{subfigure}{.5\textwidth}
		\centering
		\includegraphics[width=1.0\linewidth]{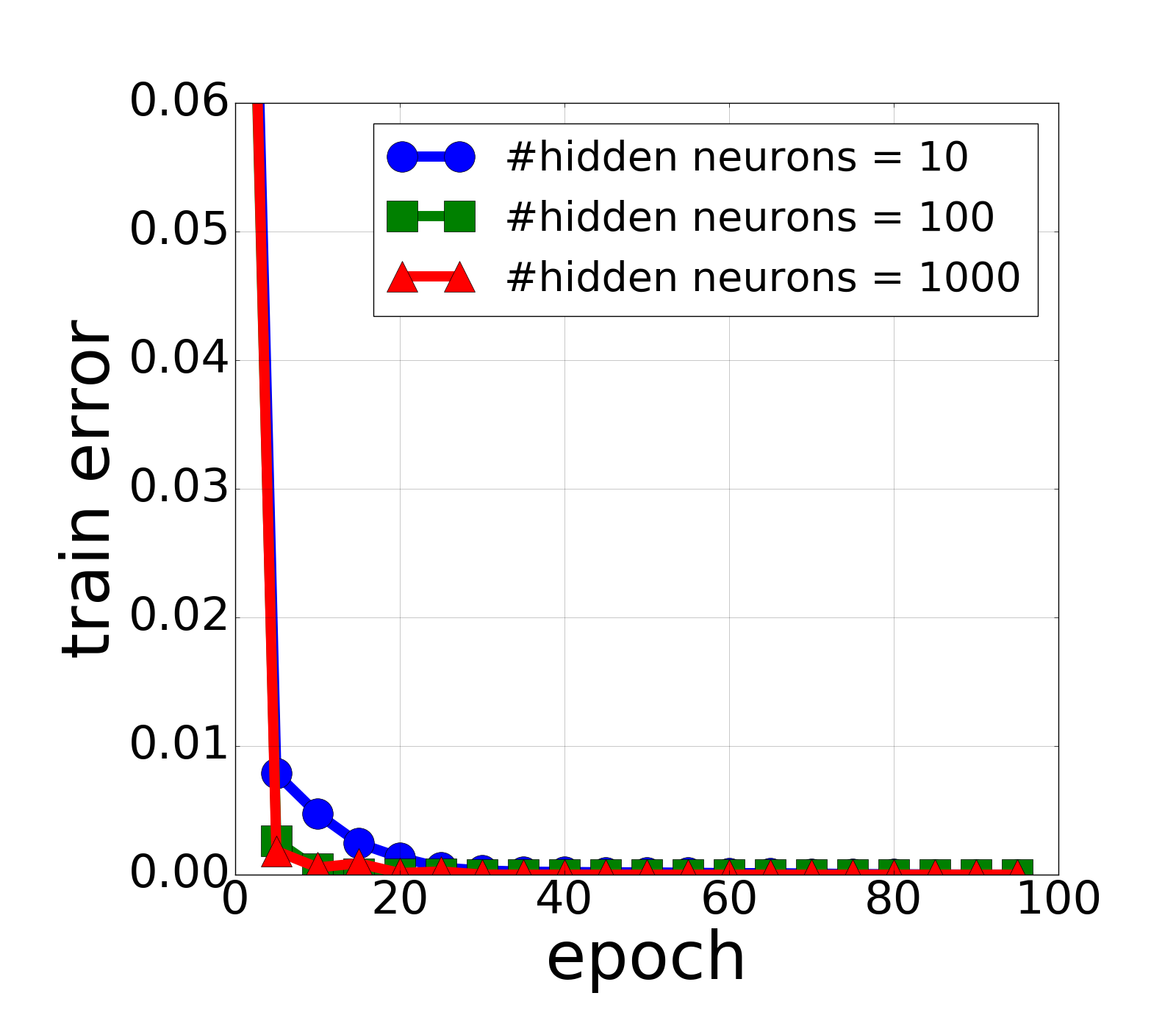}
		\caption{}
		\label{fig:sfig1}
	\end{subfigure}%
	\begin{subfigure}{.5\textwidth}
		\centering
		\includegraphics[width=1.0\linewidth]{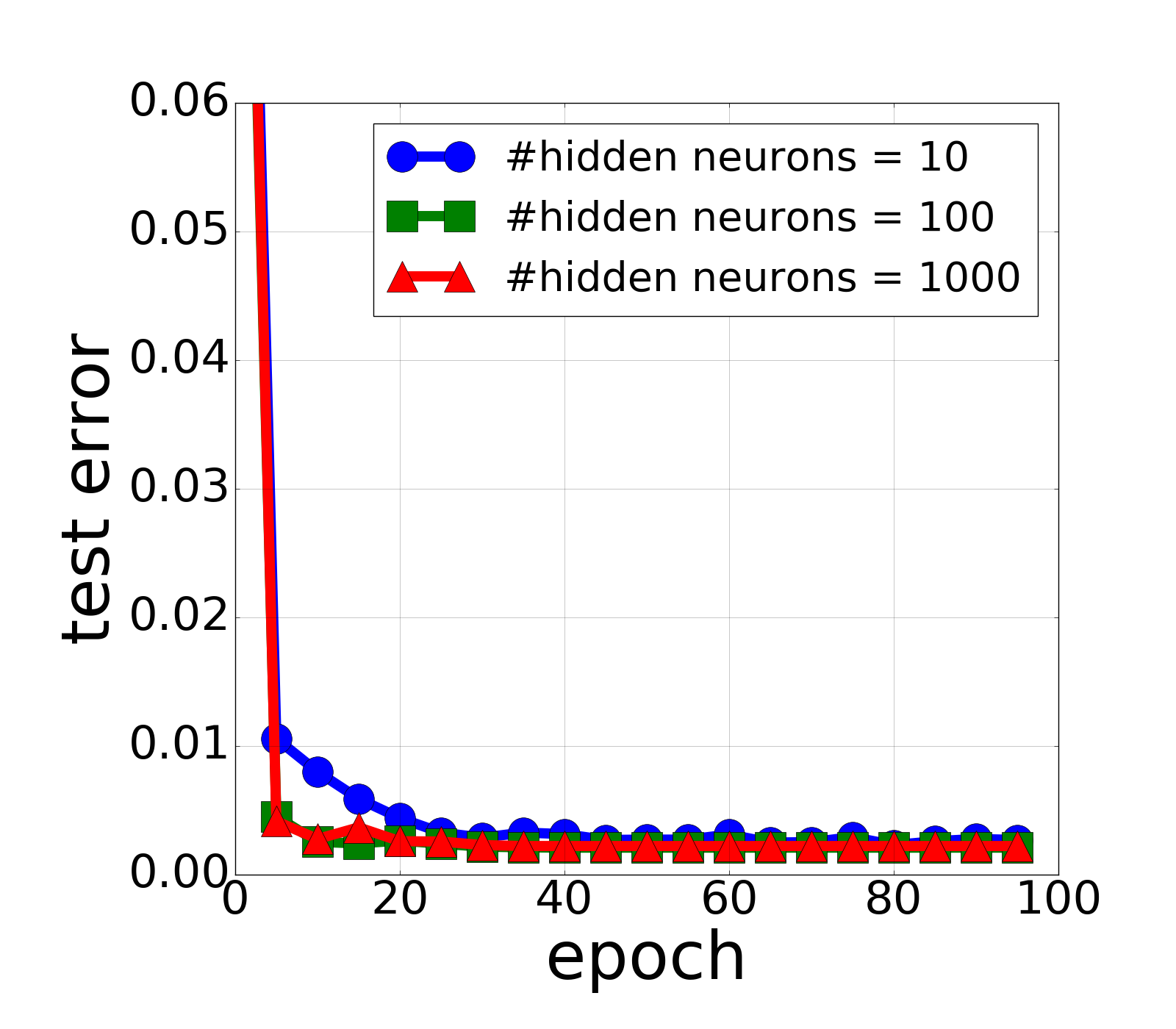}
		\caption{}
		\label{fig:sfig3}
	\end{subfigure}
	\caption{Classifying MNIST images with over-parameterized networks. The setting of Section $\ref{sec:convergence}$ is implemented (e.g., SGD with batch of size 1, only first layer is trained, Leaky ReLU activations) and SGD is initialized according to the initialization defined in \eqref{eq:nn_init}. The linearly separable data set consists of 4000 MNIST images with digits 3 and 5, each of dimension 784. The size of the training set is 3000 and the remaining 1000 points form the test set. Three experiments are performed which differ only in the number of hidden neurons, 10, 100 and 1000. In the latter two, the networks are over-parameterized. For each number of hidden neurons, 40 different runs of SGD are performed and their results are averaged. (a) shows that in all experiments SGD converges to a global minimum. (b) shows that the global minimum obtained by SGD generalizes well in all settings (including the over-parameterized).}
	\label{fig:overparam}
\end{figure}

\ignore{
In Figure \ref{fig:large_net} we show empricial results for the MNIST data set, where the goal is to classify the images with digits $3$ and $5$, which are linearly separable. The setting is the same as the network studied in previous sections - two-layer fully-connected network with Leaky ReLU activations, where here we learn a network with 100,000 hidden neurons with SGD. The experimental details are given in the Appendix. The results show that even a neural network with 100,000 hidden neurons can generalize well, as predicted by Theorem \ref{thm:generalization}. 

}

\begin{remark}
\label{rem:eta_dependence}
The generelization bound in \eqref{eq:generalization_bound} holds for $\eta \rightarrow \infty$, which is unique for the setting that we consider,
and may seem surprising, given that a choice of large $\eta$ often fails in practice. Furthermore, the bound is optimal for $\eta \rightarrow \infty$. To support this theoretical result, we show in Theorem \ref{thm:lower_bound} an example where indeed $\eta \rightarrow \infty$ is optimal in terms of the number of updates and generalization. On the other hand, we note that in practice, it may not be optimal to use large $\eta$ in our setting, since this bound results from a worst-case analysis of a sequence of examples encountered by SGD. Finally, the important thing to note is that the bound holds for any $\eta$, and is thus applicable to realistic applications of SGD.  

\end{remark}

\ignore{

\begin{figure}
	\begin{subfigure}{.33\textwidth}
		\centering
		\includegraphics[width=1.0\linewidth]{lr_train.png}
		\caption{}
		\label{fig:sfig1}
	\end{subfigure}%
	\begin{subfigure}{.33\textwidth}
		\centering
		\includegraphics[width=1.0\linewidth]{lr_updates.png}
		\caption{}
		\label{fig:sfig2}
	\end{subfigure}%
	\begin{subfigure}{.33\textwidth}
		\centering
		\includegraphics[width=1.0\linewidth]{lr_test.png}
		\caption{}
		\label{fig:sfig3}
	\end{subfigure}

	\caption{Effect of the learning rate on generalization. The setting of Section $\ref{sec:convergence}$ is implemented (e.g., SGD with batch of size 1, only first layer is trained, Leaky ReLU activations) with the linear separable data set consisting of 4000 MNIST images with digits 3 and 5. The network has 50 hidden neurons, the size of the training set is 3000 and the remaining 1000 points form the test set. Two experiments are performed which differ only in the constant learning rate, one with value 0.01 and the other with value 100. For each learning rate, 20 different executions of SGD are performed and their results are averaged. (a) shows that the training error obtained in both experiments is 0. (b) shows that for a higher learning rate there are substantially more updates. (c) shows that for the lower learning rate the generalization performance is better, which coincides with the lower number of updates and the compression bound in Theorem \ref{thm:generalization}.}
	\label{fig:large_learning_rate}
\end{figure}

}

\subsection{Expressiveness}
\label{sec:expressiveness}

Let $X \in \reals^{d \times n}$ be the matrix with the points $\xx_i$ in its columns, $\yy \in \{-1,1\}^n$ the corresponding vector of labels and let $\net(X) =  \vv^\top\sigma(WX)$ be the network defined in \eqref{eq:network} applied on the matrix $X$. By Theorem 8 in \citep{soudry2017exponentially} we immediately get the following. For completeness, the proof is given in the Appendix (Section \ref{sec:proof_thm_expressiveness}). 

\begin{theorem}\textsl{}
	\label{thm:expressiveness}
	Assume that $k \geq 2\left \lceil \frac{n}{2d-2} \right \rceil$. Then for any $\yy \in \{-1,1\}^n$ and for almost any $X$,\footnote{That is, the set of entries of $X$ which do not satisfy the statement is of Lebesgue measure 0.} there exist $\tilde{\mw} = (\tilde{W},\tilde{\vv})$ where $\tilde{W} \in \reals^{2k \times d}$ and $\tilde{\vv} = (\overbrace{\tilde{v} \dots \tilde{v}}^k, \overbrace{-\tilde{v} \dots -\tilde{v}}^k) \in \reals^{2k}$, $\tilde{v}>0$ such that $\yy = N_{\tilde{\mw}}(X)$.
\end{theorem}

Theorem \ref{thm:expressiveness} implies that for sufficiently large networks, the optimization problem (\ref{eq:optimization}) can have arbitrarely \textit{bad} global minima with respect to a given test set, i.e., ones which do not generalize well on a given test set.

\section{ReLU- Success and Failure Cases}
\label{sec:relu}
In this section we consider the same setting as in section \ref{sec:convergence}, but with the ReLU activation function $\sigma(x)=\max\{0,x\}$. In Section \ref{sec:bad_minima} we show that the loss function contains arbitrarely bad local minima. In Section \ref{sec:relu_orthogonal} we give an example where for a sufficiently small network, with high probability SGD will converge to a local minimum. On the other hand, for a sufficiently large network, with high probability SGD will converge to a global minimum.

\subsection{Existence of bad local minima}
\label{sec:bad_minima}

%This is the footnote for the figure in the inductive bias section (so that it will appear on the same page as the figure).
\footnotetext{We can only conclude that the trained network is \textit{approximately} a linear classifier because of the limited resolution of the grid.}

The result is summarized in the following theorem and the proof is deferred to the Appendix (Section \ref{sec:proof_thm_bad_minima}). The main idea is to construct a network with weight paramater $W$ such that for at least $\frac{|S|}{2}$ points $(\xx,y) \in S$ it holds that $\inner{\ww,\xx} < 0$ for each neuron with weight vector $\ww$.  Furthermore, the remaining points satisfy $y \net(\xx) > 1$ and thus the gradient is zero and $L_S(W) > \frac{1}{2}$. 

\begin{theorem}
	\label{thm:bad minima}
	Fix $\vv = (\overbrace{1 \dots 1}^k, \overbrace{-1 \dots -1}^k) \in \reals^{2k}$.	Then, for every finite set of examples $S \subseteq \mathcal{X} \times \mathcal{Y}$
	that is linearly separable, i.e., for which there exists $\ww^* \in \reals^d$
	such that for each $(\xx,y) \in S$ we have $y \inner{\ww^*, \xx} \geq 1$,
	there exists $W \in \reals^{2k \times d}$ such that $W$ is a local minimum point
	with $\loss{S} > \frac{1}{2}$.
\end{theorem}

\subsection{Orthogonal vectors - simple case analysis}
\label{sec:relu_orthogonal}

In this section we assume that $S=\{\elm _1 \dots \elm_d\}\times\{1\} \subseteq \mathcal{X} \times \mathcal{Y}$ where $\{\elm_1, \dots, \elm_d\}$ is the standard basis of $\reals^d$.
We assume all examples are labeled with the same label for simplicity, as the same result holds for the general case.
\\
Let $N_{\mw_t}$ be the network obtained at iteration $t$, where $\mw_t = (W_t,\vv)$. Assume we initialize  
with fixed $\vv = (\overbrace{1 \dots 1}^k, \overbrace{-1 \dots -1}^k)$, and
$W_0 \in \reals^{2k \times d}$ is randomly initialized from a continuous symmetric distribution
with bounded norm, i.e $|[W_0]_{i,j}| \le C$ for some $C > 0$.

The main result of this section is given in the following theorem. The proof is given in the Appendix (Section \ref{sec:proof_relu_orthogonal}). The main observation is that  the convergence to non-global minimum depends solely on the initialization and occurs if and only if there exists a point $\xx$ such that for all neurons, the corresponding initialized weight vector $\ww$ satisfies $\inner{\ww,\xx} \leq 0$.

\begin{theorem}
	\label{thm:orth_convergence}
	Fix $\delta > 0$ and assume we run SGD with examples from $S=\{\elm _1 \dots \elm_d\}\times\{1\}$.
	%to train network
	%$\mathcal{N}_{W,\vv}$ with weights initialized randomly from a symmetric distribution.
	If $k \le \log_2 (\frac{d}{-\ln(\delta)})$, then with probability of at least $1-\delta$, SGD will converge to a non global minimum point. \\
	On the other hand, if
	$k \ge \log_2 (\frac{2d}{\delta})$,
	then with probability
	of at least $1 - \delta$, SGD will converge to a global minimum point after
	$\lceil \max\{\frac{dC}{\eta},\frac{d}{\eta}\}\rceil$ iterations.
\end{theorem}

Note that in the first part of the theorem, we can make the basin of attraction of the non-global minimum exponentially large by setting $\delta = e^{-\alpha d}$ for $\alpha \leq \frac{1}{2}$. 

\section{Conclusion}
\label{sec:discussion}

Understanding the performance of over-parameterized neural networks is essential for explaining the success of deep learning models in practice. Despite a plethora of theoretical results for generalization of neural networks, none of them give guarantees for over-parameterized networks. In this work, we give the first provable guarantees for the generalization performance of over-parameterized networks, in a setting where the data is linearly separable and the network has Leaky ReLU activations. We show that SGD compresses its output when learning over-parameterized networks, and thus exhibits good generalization performance.

The analysis for networks with Leaky ReLU activations does not hold for networks with ReLU activations, since in this case the loss contains spurious local minima. However, due to the success of over-parameterized networks with ReLU activations in practice, it is likely that similar results hold here as well. It would be very interesting to provide convergence guarantees and generalization bounds for this case. Another direction for future work is to show that similar results hold under different assumptions on the data.

\subsubsection*{Acknowledgments}
%\paragraph{Acknowledgements:}
This research is supported by the Blavatnik Computer Science Research Fund and the European Research Council (TheoryDL project).

\bibliography{linsep}

\begin{thebibliography}{27}
\providecommand{\natexlab}[1]{#1}
\providecommand{\url}[1]{\texttt{#1}}
\expandafter\ifx\csname urlstyle\endcsname\relax
  \providecommand{\doi}[1]{doi: #1}\else
  \providecommand{\doi}{doi: \begingroup \urlstyle{rm}\Url}\fi

\bibitem[Anthony \& Bartlett(2009)Anthony and Bartlett]{anthony2009neural}
Anthony, Martin and Bartlett, Peter~L.
\newblock \emph{Neural network learning: Theoretical foundations}.
\newblock cambridge university press, 2009.

\bibitem[Bartlett et~al.(2017)Bartlett, Foster, and
  Telgarsky]{bartlett2017spectrally}
Bartlett, Peter, Foster, Dylan~J, and Telgarsky, Matus.
\newblock Spectrally-normalized margin bounds for neural networks.
\newblock \emph{arXiv preprint arXiv:1706.08498}, 2017.

\bibitem[Bartlett \& Mendelson(2002)Bartlett and
  Mendelson]{bartlett2002rademacher}
Bartlett, Peter~L and Mendelson, Shahar.
\newblock Rademacher and gaussian complexities: Risk bounds and structural
  results.
\newblock \emph{Journal of Machine Learning Research}, 3\penalty0
  (Nov):\penalty0 463--482, 2002.

\bibitem[Bengio(2012)]{bengio2012practical}
Bengio, Yoshua.
\newblock Practical recommendations for gradient-based training of deep
  architectures.
\newblock In \emph{Neural networks: Tricks of the trade}, pp.\  437--478.
  Springer, 2012.

\bibitem[Brutzkus \& Globerson(2017)Brutzkus and
  Globerson]{brutzkus2017globally}
Brutzkus, Alon and Globerson, Amir.
\newblock Globally optimal gradient descent for a convnet with gaussian inputs.
\newblock \emph{arXiv preprint arXiv:1702.07966}, 2017.

\bibitem[Du et~al.(2017)Du, Lee, and Tian]{du2017convolutional}
Du, Simon~S, Lee, Jason~D, and Tian, Yuandong.
\newblock When is a convolutional filter easy to learn?
\newblock \emph{arXiv preprint arXiv:1709.06129}, 2017.

\bibitem[Frasconi et~al.(1997)Frasconi, Gori, and Tesi]{frasconi1997successes}
Frasconi, P, Gori, M, and Tesi, A.
\newblock Successes and failures of backpropagation: A theoretical.
\newblock \emph{Progress in Neural Networks: Architecture}, 5:\penalty0 205,
  1997.

\bibitem[Glorot \& Bengio(2010)Glorot and Bengio]{glorot2010understanding}
Glorot, Xavier and Bengio, Yoshua.
\newblock Understanding the difficulty of training deep feedforward neural
  networks.
\newblock In \emph{Proceedings of the Thirteenth International Conference on
  Artificial Intelligence and Statistics}, pp.\  249--256, 2010.

\bibitem[Gori \& Tesi(1992)Gori and Tesi]{gori1992problem}
Gori, Marco and Tesi, Alberto.
\newblock On the problem of local minima in backpropagation.
\newblock \emph{IEEE Transactions on Pattern Analysis and Machine
  Intelligence}, 14\penalty0 (1):\penalty0 76--86, 1992.

\bibitem[Hardt et~al.(2016)Hardt, Ma, and Recht]{hardt2016gradient}
Hardt, Moritz, Ma, Tengyu, and Recht, Benjamin.
\newblock Gradient descent learns linear dynamical systems.
\newblock \emph{arXiv preprint arXiv:1609.05191}, 2016.

\bibitem[Kawaguchi et~al.(2017)Kawaguchi, Kaelbling, and
  Bengio]{kawaguchi2017generalization}
Kawaguchi, Kenji, Kaelbling, Leslie~Pack, and Bengio, Yoshua.
\newblock Generalization in deep learning.
\newblock \emph{arXiv preprint arXiv:1710.05468}, 2017.

\bibitem[Kuzborskij \& Lampert(2017)Kuzborskij and Lampert]{kuzborskij2017data}
Kuzborskij, Ilja and Lampert, Christoph.
\newblock Data-dependent stability of stochastic gradient descent.
\newblock \emph{arXiv preprint arXiv:1703.01678}, 2017.

\bibitem[Li \& Yuan(2017)Li and Yuan]{li2017convergence}
Li, Yuanzhi and Yuan, Yang.
\newblock Convergence analysis of two-layer neural networks with relu
  activation.
\newblock \emph{arXiv preprint arXiv:1705.09886}, 2017.

\bibitem[Littlestone \& Warmuth(1986)Littlestone and
  Warmuth]{littlestone1986relating}
Littlestone, Nick and Warmuth, Manfred.
\newblock Relating data compression and learnability.
\newblock Technical report, Technical report, University of California, Santa
  Cruz, 1986.

\bibitem[Neyshabur et~al.(2014)Neyshabur, Tomioka, and
  Srebro]{neyshabur2014search}
Neyshabur, Behnam, Tomioka, Ryota, and Srebro, Nathan.
\newblock In search of the real inductive bias: On the role of implicit
  regularization in deep learning.
\newblock \emph{arXiv preprint arXiv:1412.6614}, 2014.

\bibitem[Neyshabur et~al.(2015)Neyshabur, Tomioka, and
  Srebro]{neyshabur2015norm}
Neyshabur, Behnam, Tomioka, Ryota, and Srebro, Nathan.
\newblock Norm-based capacity control in neural networks.
\newblock In \emph{Conference on Learning Theory}, pp.\  1376--1401, 2015.

\bibitem[Neyshabur et~al.(2017{\natexlab{a}})Neyshabur, Bhojanapalli,
  McAllester, and Srebro]{neyshabur2017exploring}
Neyshabur, Behnam, Bhojanapalli, Srinadh, McAllester, David, and Srebro,
  Nathan.
\newblock Exploring generalization in deep learning.
\newblock \emph{arXiv preprint arXiv:1706.08947}, 2017{\natexlab{a}}.

\bibitem[Neyshabur et~al.(2017{\natexlab{b}})Neyshabur, Bhojanapalli,
  McAllester, and Srebro]{neyshabur2017pac}
Neyshabur, Behnam, Bhojanapalli, Srinadh, McAllester, David, and Srebro,
  Nathan.
\newblock A pac-bayesian approach to spectrally-normalized margin bounds for
  neural networks.
\newblock \emph{arXiv preprint arXiv:1707.09564}, 2017{\natexlab{b}}.

\bibitem[Nguyen \& Hein(2017)Nguyen and Hein]{nguyen2017loss}
Nguyen, Quynh and Hein, Matthias.
\newblock The loss surface of deep and wide neural networks.
\newblock \emph{arXiv preprint arXiv:1704.08045}, 2017.

\bibitem[Shalev-Shwartz \& Ben-David(2014)Shalev-Shwartz and
  Ben-David]{shalev2014understanding}
Shalev-Shwartz, Shai and Ben-David, Shai.
\newblock \emph{Understanding machine learning: From theory to algorithms}.
\newblock Cambridge university press, 2014.

\bibitem[Soltanolkotabi et~al.(2017)Soltanolkotabi, Javanmard, and
  Lee]{soltanolkotabi2017theoretical}
Soltanolkotabi, Mahdi, Javanmard, Adel, and Lee, Jason~D.
\newblock Theoretical insights into the optimization landscape of
  over-parameterized shallow neural networks.
\newblock \emph{arXiv preprint arXiv:1707.04926}, 2017.

\bibitem[Soudry \& Hoffer(2017)Soudry and Hoffer]{soudry2017exponentially}
Soudry, Daniel and Hoffer, Elad.
\newblock Exponentially vanishing sub-optimal local minima in multilayer neural
  networks.
\newblock \emph{arXiv preprint arXiv:1702.05777}, 2017.

\bibitem[Tian(2017)]{tian2017analytical}
Tian, Yuandong.
\newblock An analytical formula of population gradient for two-layered relu
  network and its applications in convergence and critical point analysis.
\newblock \emph{arXiv preprint arXiv:1703.00560}, 2017.

\bibitem[Xu \& Mannor(2012)Xu and Mannor]{xu2012robustness}
Xu, Huan and Mannor, Shie.
\newblock Robustness and generalization.
\newblock \emph{Machine learning}, 86\penalty0 (3):\penalty0 391--423, 2012.

\bibitem[Yu \& Chen(1995)Yu and Chen]{yu1995local}
Yu, Xiao-Hu and Chen, Guo-An.
\newblock On the local minima free condition of backpropagation learning.
\newblock \emph{IEEE Transactions on Neural Networks}, 6\penalty0 (5):\penalty0
  1300--1303, 1995.

\bibitem[Zhang et~al.(2016)Zhang, Bengio, Hardt, Recht, and
  Vinyals]{zhang2016understanding}
Zhang, Chiyuan, Bengio, Samy, Hardt, Moritz, Recht, Benjamin, and Vinyals,
  Oriol.
\newblock Understanding deep learning requires rethinking generalization.
\newblock \emph{arXiv preprint arXiv:1611.03530}, 2016.

\bibitem[Zhong et~al.(2017)Zhong, Song, Jain, Bartlett, and
  Dhillon]{zhong2017recovery}
Zhong, Kai, Song, Zhao, Jain, Prateek, Bartlett, Peter~L, and Dhillon,
  Inderjit~S.
\newblock Recovery guarantees for one-hidden-layer neural networks.
\newblock \emph{arXiv preprint arXiv:1706.03175}, 2017.

\end{thebibliography}
\bibliographystyle{icml2017}

\appendix
%\renewcommand{\thesubsection}{\Alph{subsection}}
%\tableofcontents

\section{Appendix}
\label{sec:appendix}

\subsection{Missing Proofs for Section \ref{sec:convergence}}

\subsubsection{Proof of Proposition \ref{prop:loss_properties}}

	\begin{enumerate}
		\item Denote by $\vec{W} = \left(\wvec{1} \dots \wvec{k} \uvec{1} \dots \uvec{k} \right) \in \reals^{2kd}$ the vector of all parameters where each $\wvec{i},\uvec{i} \in \mathbb{R}^d$. Let $(\xx,y) \in S$, then if $y\net{}(\xx) < 1$, it holds that $$\inner{\frac{\partial}{\partial \wvec{i}}L_{\{(\xx,y)\}}(\vec{W}),\ww^*} = \inner{y\sigma'\left(\inner{\wvec{i},\xx}\right) \xx, \ww^*} \geq \sigma'\left(\inner{\wvec{i},\xx}\right) > 0$$
		and similarly,
		$$\inner{\frac{\partial}{\partial \uvec{i}}L_{\{(\xx,y)\}}(\vec{W}),-\ww^*} = \inner{-y\sigma'\left(\inner{\uvec{i},\xx}\right) \xx, -\ww^*} \geq \sigma'\left(\inner{\wvec{i},\xx}\right) > 0 \text{.}$$
		
		Hence if we define $\vec{W}^* = (\overbrace{\ww^* \dots \ww^*}^k, \overbrace{-\ww^* \dots -\ww^*}^k) \in \reals^{2kd}$, then $$\inner{\frac{\partial}{\partial \vec{W}}L_{\{(\xx,y)\}}(\vec{W}),\vec{W}^*} > 0$$

		Otherwise, if $y\net{}(\xx) \geq 1$, then the gradient vanishes and thus $$\inner{\frac{\partial}{\partial \vec{W}}L_{\{(\xx,y)\}}(\vec{W}),\vec{W}^*} = 0$$
		
		It follows that if there exists $(\xx,y) \in S$, such that $y\net{}(\xx) < 1$, then we have $$\inner{\frac{\partial}{\partial \vec{W}}L_{S}(\vec{W}),\vec{W}^*} = \frac{1}{n}\sum_{i=1}^{n}{\inner{\frac{\partial}{\partial \vec{W}}L_{\{(\xx_i,y_i)\}}(\vec{W}),\vec{W}^*}} > 0$$ and thus $\frac{\partial}{\partial \vec{W}}L_{S}(\vec{W}) \neq 0$. Therefore, for any critical point it holds that $y\net{}(\xx) \geq 1$ for all $(\xx,y) \in S$, which implies that it is a global minimum.
		\item For simplicity consider the function $f_{\xx}(\ww,\uu) = \sigma(\inner{\ww,\xx}) - \sigma(\inner{\uu,\xx})$ for $\xx \ne 0$. Define $\ww_1=\ww_2=\uu_1=\xx$ and  $\uu_2 = -\xx$. Then $$f_{\xx}(\ww_1,\uu_1) = 0$$
		$$f_{\xx}(\ww_2,\uu_2) = (1+\alpha)\norm{\xx}^2 $$
		and
		$$f_{\xx}(\frac{\ww_1 + \ww_2}{2},\frac{\uu_1 + \uu_2}{2}) = \norm{\xx}^2 $$
		and thus $f_{\xx}(\frac{\ww_1 + \ww_2}{2},\frac{\uu_1 + \uu_2}{2}) > \frac{1}{2}f_{\xx}(\ww_1,\uu_1) + \frac{1}{2}f_{\xx}(\ww_2,\uu_2)$ which implies that the function is not convex.
	\end{enumerate}

\subsubsection{Proof of Theorem \ref{thm:convergence}}
\label{sec:thm_convergence_proof}
	%TODO provide proof for the assumption
	Assume SGD performed $t$ non-zero updates. We will show that $t\leq M_k$. We note that if there is no $(\xx,y) \in S$ such that the corresponding update is non-zero, then SGD has reached a critical point of $L_S$ (which is a global minimum by Proposition \ref{prop:loss_properties}).
	Let $\vec{W}^* = (\overbrace{\ww^* \dots \ww^*}^k, \overbrace{-\ww^* \dots -\ww^*}^k) \in \reals^{2kd}$ and note that $L_S(\vec{W^*})=0$, i.e., $\vec{W^*}$ is a global minimum. Define the following two functions:
	\[
	F(W_t) = \inner{\vec{W}_t, \vec{W}^*} = \sum_{i=1}^k \inner{\wvec{i}_t,\ww^*} -
	\sum_{i=1}^k \inner{\uvec{i}_t,\ww^*}
	\]
	\[
	G(W_t) = \norm{\vec{W}_t} = \sqrt{\sum_{i=1}^k \norm{\wvec{i}_t}^2 +
		\sum_{i=1}^k \norm{\uvec{i}_t}^2}
	\]
	Then, from Cauchy-Schwartz inequality we have
	\begin{equation}
	\label{cs-eq}
	\frac{\left|F(W_t)\right|}{G(W_t)\norm{\vec{W}^*}}  = 
	\frac{\left|\inner{\vec{W}_t, \vec{W}^*}\right|}{\norm{\vec{W}_t}\norm{\vec{W}^*}}
	\le 1
	\end{equation}
	
	%TODO ref
	Since the update at iteration $t$ is non-zero, we have $y_tN_{t-1}(\xx_t) < 1$ and the update rule is given by 
	\begin{equation}
	\label{eq:gradient_updates}
	%	\begin{split}
	\wvec{i}_t = \wvec{i}_{t-1} + \eta v\pti y_t\xx_t   \ \ , \ \ 
	\uvec{i}_t = \uvec{i}_{t-1} - \eta v\qti  y_t\xx_t
	%	\end{split}
	\end{equation}
	where $\pti = 1$ if $\inner{\wvec{i}_{t-1},\xx_t} \geq 0$ and $\pti = \alpha$ otherwise. Similarly $\qti = 1$ if $\inner{\uvec{i}_{t-1},\xx_t} \geq 0$ and $\qti = \alpha$ otherwise.
	It follows that:
	\begin{align*}
	G(W_t)^2 & = \sum_{i=1}^k \norm{\wvec{i}_t}^2 +
	\sum_{i=1}^k \norm{\uvec{i}_t}^2 \\
	& \leq \sum_{i=1}^k \norm{\wvec{i}_{t-1}}^2 +
	\sum_{i=1}^k \norm{\uvec{i}_{t-1}}^2 +
	2 \eta v y_t \left( \sum_{i=1}^k \inner{\wvec{i}_{t-1},\xx_t}\pti
	-\sum_{i=1}^k \inner{\uvec{i}_{t-1},\xx_t}\qti \right) +
	2k \eta^2 v^2 \norm{\xx_t}^2 \\
	& < \sum_{i=1}^k \norm{\wvec{i}_{t-1}}^2 +
	\sum_{i=1}^k \norm{\uvec{i}_{t-1}}^2 + 2 \eta + 2k \eta^2 v^2 =
	G(W_{t-1})^2 + 2 \eta + 2k \eta^2 v^2 
	\end{align*}
	where the second inequality follows since $y_t v\left( \sum_{i=1}^k \inner{\wvec{i}_{t-1},\xx_t}\pti
	-\sum_{i=1}^k \inner{\uvec{i}_{t-1},\xx_t}\qti \right) =  y_tN_{t-1}(\xx_t) < 1$. Using the above recursively, we obtain:
	\begin{equation}
	\label{G_equation}
	G(W_t)^2 \le G(W_0)^2 +  t(2k \eta^2v^2+2 \eta) 
	\end{equation}
	
	On the other hand,
	\begin{align*}
	F(W_t) & = \sum_{i=1}^k \inner{\wvec{i}_t,\ww^*} -
	\sum_{i=1}^k \inner{\uvec{i}_t,\ww^*} \\
	& = \sum_{i=1}^k \inner{\wvec{i}_{t-1},\ww^*} -
	\sum_{i=1}^k \inner{\uvec{i}_{t-1},\ww^*}
	+ \eta v \sum_{i=1}^k \inner{y_t \xx_t,\ww^*} \pti +
	\eta v \sum_{i=1}^k \inner{y_t \xx_t,\ww^*} \qti \\
	& \ge \sum_{i=1}^k \inner{\wvec{i}_{t-1},\ww^*} -
	\sum_{i=1}^k \inner{\uvec{i}_{t-1},\ww^*} + 2k \eta v\alpha = 
	F(W_{t-1}) + 2k \eta v\alpha
	\end{align*}
	where the inequality follows since $ \inner{y_t \xx_t,\ww^*} \geq 1 $. This implies that 
	\begin{equation}
	\label{F_equation}
	F(W_t) \ge F(W_0) + 2k \eta v\alpha t 
	\end{equation}
	
	By combining equations \eqref{cs-eq}, \eqref{G_equation} and \eqref{F_equation} we get,
	\begin{align*}
	-G(W_0) \norm{\vec{W}^*} + 2k \eta v \alpha t &
	\le F(W_0) + 2k \eta v \alpha t \le F(W_t) \le \norm{\vec{W}^*} G(W_t)  \\
	& \le \norm{\vec{W}^*} \sqrt{G(W_0)^2 + t(2k \eta^2 v^2+2 \eta)}
	\end{align*}
	Using $\sqrt{a + b} \le \sqrt{a} + \sqrt{b}$ the above implies,
	$$-G(W_0) \norm{\vec{W}^*} + 2k \eta v\alpha t \leq \norm{\vec{W}^*} G(W_0) +  \norm{\vec{W}^*} \sqrt{t}\sqrt{2k \eta^2v^2+2 \eta} $$
	
	%TODO add to assumptions
	Since
	$\norm{\wvec{i}_0}, \norm{\uvec{i}_0} \le R$ we have $G(W_0) \le \sqrt{2k}R$. Noting that 
	$  \norm{\vec{W}^*} = \sqrt{2k}  \norm{\ww^*}$ we get,
	$$a t \leq b\sqrt{t} + c$$ where $a = 2k \eta v\alpha $, $b = \sqrt{(4k^2 \eta^2v^2+4 \eta k)}\norm{\ww^*}$ and $c = 4kR\norm{\ww^*}$. By inspecting the roots of the parabola $P(x)=x^2-\frac{b}{a}x-\frac{c}{a}$ we conclude that
	\begin{equation}
	\begin{split}
	t \leq \Big(\frac{b}{a}\Big)^2 + \sqrt{\frac{c}{a}}\frac{b}{a} + \frac{c}{a} &= \frac{(4k^2 \eta^2v^2+4 \eta k)\norm{\ww^*}^2}{4k^2\eta^2v^2 \alpha^2} + \frac{\sqrt{(4k^2 \eta^2 v^2+4 \eta k)}\norm{\ww^*}}{2k \eta v \alpha}\sqrt{\frac{2R\norm{\ww^*}}{\eta v\alpha}} + \frac{2R\norm{\ww^*}}{\eta v\alpha} \\ &= \frac{\norm{\ww^*}^2}{\alpha^2}   + \frac{\norm{\ww^*}^2}{k\eta v^2\alpha^2} + \frac{\sqrt{R(8k^2 \eta^2 v^2+ 8\eta k)}{\norm{\ww^*}}^{1.5}}{2k (\eta v\alpha)^{1.5}} +\frac{2R\norm{\ww^*}}{\eta v\alpha} = M_k
	\end{split}
	\end{equation}

\subsubsection{Proof of Corollary \ref{cor:convergence}}
\label{sec:cor_convergence_proof}

Since $\frac{R}{v} = 1$, we have by Theorem \ref{thm:convergence} and the inequality $\sqrt{a + b} \le \sqrt{a} + \sqrt{b}$,

\begin{equation}
\begin{split}
M_k &= \frac{\norm{\ww^*}^2}{\alpha^2} + O\Bigg(\frac{\norm{\ww^*}^{2}}{\eta}\Bigg) + O\Bigg(\frac{\norm{\ww^*}^{1.5}}{\sqrt{\eta}}\Bigg) +  O\Bigg(\frac{\norm{\ww^*}^{1.5}}{\eta}\Bigg) + O\Bigg(\frac{\norm{\ww^*}}{\eta}\Bigg) \\ &= \frac{\norm{\ww^*}^2}{\alpha^2} + O\Bigg(\frac{\norm{\ww^*}^2}{\min\{\eta,\sqrt{\eta}\}}\Bigg) \text{.}
\end{split}
\end{equation}

\subsubsection{Proof of Theorem \ref{thm:lower_bound}}
\label{sec:proof_lower_bound}
We will prove a more general theorem. Theorem \ref{thm:lower_bound} follows by setting $R=v=\frac{1}{\sqrt{2k}}$.

\begin{theorem}
	\label{thm:lower_bound_general}
	For any $d$ there exists a sequence of linearly separable points on which SGD will make at least $$\max\Big\{\min\big\{B_1, B_2 \big\}, \norm{\ww^*}^2\Big\}$$ updates, where $$B_1=\frac{R\norm{\ww^*}}{\eta v \alpha } + \min\left\lbrace\frac{\norm{\ww^*}^2}{2\eta kv^2} - \alpha\norm{\ww^*},0\right\rbrace$$ and $$B_2 = \frac{R\norm{\ww^*}}{\eta v } + \min\left\lbrace\frac{\norm{\ww^*}^2}{2\alpha^2 \eta kv^2} - \frac{\norm{\ww^*}}{\alpha}, 0\right\rbrace$$ 
\end{theorem}
\begin{proof}
Define a sequence $\ms$ of size $d$, $$(\elm_1,1),(\elm_2,1),...,(\elm_d,1)$$ where $\{\elm_i\}$ is the standard basis of $\reals^d$ and let $\ww^*=(1,1,...,1) \in \reals^d$. Note that $d=\norm{\ww^*}^2$ and $\inner{\ww^*,\elm_i} \geq 1$ for all $1 \leq i \leq d$. We will consider the case where SGD runs on a sequence of examples which consists of multiple copies of $\ms$ one after the other.

Assume SGD is initialized with $$\wvec{i}_0 = -\sum_{j=1}^{d}{\frac{R}{\sqrt{d}}\elm_j}$$ $$\uvec{i}_0 = \sum_{j=1}^{d}{\frac{R}{\sqrt{d}}\elm_j}$$

for all $1 \leq i \leq k$. Note that $\norm{\wvec{i}_0}, \norm{\uvec{i}_0} \le R$ for all $1 \leq i \leq k$. 

Since $\wvec{i}_0=\wvec{j}_0$ and $\uvec{i}_0=\uvec{j}_0$ for all $i\neq j$, we have by induction that $\wvec{i}_t=\wvec{j}_t$ and $\uvec{i}_t=\uvec{j}_t$ for all $i\neq j$ and $t > 0$. Hence, we will denote $\ww_t = \wvec{i}_t$ and $\uu_t = \uvec{i}_t$ for all $1 \leq i \leq d$ and $t>0$. Then for all $1 \leq j \leq d$ we have $N_{\mw_t}(\elm_j)= kv\sigma'(\inner{\ww_t,\elm_j})\inner{\ww_t,\elm_j} - kv\sigma'(\inner{\uu_t,\elm_j})\inner{\uu_t,\elm_j}$.

Since at the global minumum $\mn_{W_t,\vv}(\elm_j) \geq 1$ for all $1 \leq j \leq d$, it follows that a necessary condition for convergence to a global minimum is that there exists an iteration $t$ in which either $kv\sigma'(\inner{\ww_t,\elm_d})\inner{\ww_t,\elm_d} \geq \frac{1}{2}$ or $-kv\sigma'(\inner{\uu_t,\elm_d})\inner{\uu_t,\elm_d} \geq \frac{1}{2}$. Equivalently, either $\inner{\ww_t,\elm_d} \geq \frac{1}{2kv}$ or $\inner{\uu_t,\elm_d} \leq - \frac{1}{2\alpha kv}$.

Since $\inner{\ww_0,\elm_d} = -\frac{R}{\sqrt{d}}$, then by the gradient updates (\eqref{eq:gradient_updates}) it follows that after at least $\frac{R}{\eta v \alpha \sqrt{d}}$ copies of $\ms$, or equivalently, after at least $\frac{Rd}{\eta v \alpha \sqrt{d}}$ iterations we will have $0 \leq \inner{\ww_t,\elm_d} \leq \eta v \alpha$. Then, after at least $\frac{d\min\{\frac{1}{2kv} - \eta v \alpha,0\}}{\eta v}$ iterations we have $\inner{\ww_t,\elm_d} \geq \frac{1}{2kv}$. Thus, in total, after at least $ \frac{R\norm{\ww^*}}{\eta v \alpha } + \min\{\frac{\norm{\ww^*}^2}{2\eta kv^2} - \alpha\norm{\ww^*}^2, 0\}$ iterations, we have $\inner{\ww_t,\elm_d} \geq \frac{1}{2kv}$.

By the same reasoning, we have $\inner{\uu_t,\elm_d} \leq -\frac{1}{2\alpha kv}$ after at least $ \frac{R\norm{\ww^*}}{\eta v } + \min\{\frac{\norm{\ww^*}^2}{2\alpha^2 \eta kv^2} - \frac{\norm{\ww^*}^2}{\alpha}, 0\}$ iterations. Finally, SGD must update on at least $d$ points in order to converge to the global minimum. The claim now follows.		
\end{proof}

\subsection{Missing Proofs for Section \ref{sec:generalization}}

\subsubsection{Proof of Theorem \ref{thm:leaky_generalization}}
\label{sec:proof_thm_generalization}
By Theorem 30.2 and Corollary 30.3 in \cite{shalev2014understanding}, for $n \geq 2c_k$ we have that with probability of at least $1-\delta$ over the choice of $S$ 
\begin{equation}
\label{eq:gen_bound_training_set_only}
L_{\mathcal{D}}(SGD_k(S,W_0)) \leq L_{V}(SGD_k(S,W_0)) + \sqrt{L_{V}(SGD_k(S,W_0)) \frac{4c_k\log{\frac{n}{\delta}}}{n}} + \frac{8c_k\log{\frac{n}{\delta}}}{n}
\end{equation}

The above result holds for a \textit{fixed} initialization $W_0$. We will show that the same result holds with high probability over $S$ \textit{and} $W_0$, where $W_0$ is chosen independently of $S$ and satisfies \eqref{eq:initialization}.
Define $\cB$ to be the event that the inequality \eqref{eq:gen_bound_training_set_only} does not hold. Then we know that $\prob_S(\cB | W_0) \leq \delta$ for any fixed initialization $W_0$. \footnote{This is where we use the independence assumption on $S$ and $W_0$ . In the proof of Theorem 30.2 in \cite{shalev2014understanding}, the hypothesis $h_I$ needs to be independent of $V$. Our independence assumption ensures that this holds.} Hence, by the law of total expectation, $$\prob_{S,W_0}(\cB) = \expect{W_0}{\prob_{S}(\cB|W_0)} \leq \delta$$

\ignore{
\subsubsection{Proof of Lemma \ref{lem:inductive_bias}}
\label{sec:proof_lemma_inductive_bias}

Notice that $f(t)$ is of the form $f(t) = \frac{-a + bt}{\sqrt{a^2 + ct}}$ where $a,b,c > 0$. We have, $$f'(t) = \frac{b\sqrt{a^2+ct} - \frac{c}{2\sqrt{a^2+ct}}(-a+bt)}{a^2+ct} = \frac{bct + 2a^2b +ac}{2(a^2+ct)^{\frac{3}{2}}}$$

and thus $f'(t) > 0$ for $t > 0$. Finally, $\lim_{t \rightarrow \infty}{f(t)} = \infty$ is immediate. 
}

\subsubsection{Proof of Theorem \ref{thm:expressiveness}}
\label{sec:proof_thm_expressiveness}
We can easily extend Theorem 8 in \citep{soudry2017exponentially} to hold for labels in $\{-1,1\}$. By the theorem we can construct networks $N_{\mw_1}$ and $N_{\mw_2}$ such that for all $i$:
\begin{enumerate}
	\item $N_{\mw_1}(\xx_i) = 1$ if $y_i=1$ and $N_{\mw_1}(\xx_i) = 0$ otherwise.
	\item $N_{\mw_2}(\xx_i) = 1$ if $y_i=-1$ and $N_{\mw_2}(\xx_i) = 0$ otherwise.
\end{enumerate}
Then $(N_{\mw_1}-N_{\mw_2})(\xx_i)=y_i$ and $N_{\mw_1}-N_{\mw_2} = \mathcal{N}_{\tilde{\mw}}$ for $\tilde{\mw} = (\tilde{W},\tilde{v})$ where $\tilde{W} \in \reals^{2k \times d}$ and $\tilde{\vv} = (\overbrace{\tilde{v} \dots \tilde{v}}^k, \overbrace{-\tilde{v} \dots -\tilde{v}}^k) \in \reals^{2k}$, $\tilde{v} > 0$.

\subsection{Missing Proofs for Section \ref{sec:relu}}

\subsubsection{Proof of Theorem \ref{thm:bad minima}}
\label{sec:proof_thm_bad_minima}
We first need the following lemma.

\begin{lemma}
	\label{lem:orthvec}
	There exists $\hat{\ww} \in \reals^d$ that satisfies the following:
	\begin{enumerate}
		\item There exists $\alpha > 0$ such that for each $(\xx,y) \in S$
		we have $|\inner{\xx,\hat{\ww}}| > \alpha$.
		\item $\#\{(\xx,y) \in S: ~ \inner{\hat{\ww},\xx} < 0 \} > \frac{1}{2} |S|$. \label{asmpt:setsize}
	\end{enumerate}
\end{lemma}
\begin{proof}
	Consider the set $V = \{\vv \in \reals^d: ~ \exists_{(\xx,y) \in S} \inner{\vv,\xx} = 0\}$.
	Clearly, $V$ is a finite union of hyper-planes
	and therefore has measure zero, so there exists $\hat{\ww} \in \reals^d \setminus V$.
	Let $\beta = \min_{(x,y) \in S} \{ \left|\inner{\hat{\ww}, \xx} \right|\}$, and since $S$
	is finite we clearly have $\alpha > 0$. Finally, if 
	\[\#\{(\xx,y) \in S: ~ \inner{\hat{\ww},\xx} < 0 \} > \frac{1}{2} |S|\]
	we can choose $\hat{\ww}$ and $\alpha=\frac{\beta}{2}$ and we are done.
	Otherwise, choosing $-\hat{\ww}$ and $\alpha=\frac{\beta}{2}$ satisfies all the
	assumptions of the lemma.
\end{proof}

We are now ready to prove the theorem. Choose $\hat{\ww} \in \reals^d$ that satisfies the
assumptions in Lemma \ref{lem:orthvec}.
Now, let $c > \frac{\norm{\ww^*}}{\alpha}$, and let $\ww = c \hat{\ww} + \ww^*$
and $\uu = c \hat{\ww} -\ww^*$. Define
\[
W = [\overbrace{\ww \dots \ww}^k, \overbrace{\uu \dots \uu}^k]^\top \in \reals^{2k \times d}
\]
Let $(\xx,y) \in S$ be an arbitrary example. \\
If $\inner{\hat{\ww},\xx} > \alpha$, then
\begin{align*}
&\inner{\ww,\xx} = c \inner{\hat{\ww},\xx} + \inner{\ww^*,\xx} \geq c \alpha - \norm{\ww^*} > 0 \\
&\inner{\uu,\xx} = c \inner{\hat{\ww},\xx} - \inner{\ww^*,\xx} \geq c \alpha - \norm{\ww^*} > 0
\end{align*}
It follows that
\begin{align*}
\net(\xx) &=
\sum_1^k \sigma(\inner{\ww,\xx}) - \sum_1^k \sigma(\inner{\uu,\xx}) \\
&= \sum_1^k (c \inner{\hat{\ww},\xx} + \inner{\ww^*,\xx}) 
- \sum_1^k (c \inner{\hat{\ww},\xx} - \inner{\ww^*,\xx}) \\ &= 2k \inner{\ww^*,\xx}
\end{align*}
Therefore $y \net(\xx) > 1$, so we get zero loss for this example,
and therefore the gradient of the loss will also be zero. \\
If, on the other hand, $\inner{\hat{\ww},\xx} < -\alpha$, then
\begin{align*}
&\inner{\ww,\xx} = c \inner{\hat{\ww},\xx} + \inner{\ww^*,\xx} \leq -c \alpha + \norm{\ww^*} < 0 \\
&\inner{\uu,\xx} = c \inner{\hat{\ww},\xx} - \inner{\ww^*,\xx} \leq -c \alpha + \norm{\ww^*} < 0
\end{align*}
and therefore
\begin{align*}
\net(\xx) &=
\sum_1^k \sigma(\inner{\ww,\xx}) - \sum_1^k \sigma(\inner{\uu,\xx}) = 0 \text{.}
\end{align*}
In this case the loss on the example would be $\max\{1-y\net(\xx),0\} = 1$, but the gradient will also be zero. Along with assumption \ref{asmpt:setsize},
we would conclude that:
\[
\loss{S} > \frac{1}{2}, ~
\frac{\partial}{\partial W} \loss{S} = 0
\]
Notice that since all the inequalities are strong, the following holds for all
$W^\prime \in \reals^{2k \times d}$ that satisfies $\norm{W^\prime - W} < \epsilon$,
for a small enough $\epsilon > 0$. Therefore, $W \in \reals^{2k \times d}$ is
indeed a local minimum.
%TODO more formally?

\subsubsection{Proof of Theorem \ref{thm:orth_convergence}}
\label{sec:proof_relu_orthogonal}
Denote $W_t = [\wvec{1}_t \dots \wvec{k}_t \uvec{1}_t \dots \uvec{k}_t]$ and define $K_t = \{\elm_j: ~ \forall_{i \in [k]} \inner{\wvec{i}_t, \elm_j} \le 0 \}$. We first prove the following lemma.
\begin{lemma}
	\label{lem:orth_minimum}
	For every $t$ we get $K_{t+1} = K_t$.
\end{lemma}
\begin{proof}
	Let $\elm_j$ be the example seen in time $t$.
	If $N_{\mw_t} (\elm_j) \geq 1$ then there is no update and we are done. Otherwise,
	if $\elm_j \in K_t$ then for each $i \in [k]$ we have 
	$\frac{\partial}{\partial \wvec{i}_t} N_{\mw_t} (\elm_j) = 0$
	and therefore the update does not change the value of $\wvec{i}_t$,
	and thus $K_{t+1} = K_t$.
	If $\elm_j \notin K_t$ then there exists $i \in [k]$
	such that $\inner{\wvec{i}_t,\elm_j} > 0$. In that case,
	we update $\wvec{i}_{t+1} \leftarrow \wvec{i}_t + \eta \elm_j$.
	Now, note that
	\[
	\inner{\wvec{i}_{t+1},\elm_j} =
	\inner{\wvec{i}_t,\elm_j} + \eta \inner{\elm_j, \elm_j} > \inner{\wvec{i}_t,\elm_j} > 0
	\]
	and therefore $\elm_j \notin K_{t+1}$. Furthermore, for each $\elm_\ell$ where $\ell \neq j$, by
	the orthogonality of the vectors we know that for each $i \in [k]$ it holds that
	\[
	\inner{\wvec{i}_{t+1},\elm_\ell} =
	\inner{\wvec{i}_t,\elm_\ell} + \eta \inner{\elm_j, \elm_\ell} = \inner{\wvec{i}_t,\elm_\ell}
	\]
	Thus $\elm_\ell \in K_{t}$ if and only if $\elm_\ell \in K_{t+1}$ and this concludes the lemma.
\end{proof}

We can now prove the theorem. For each $j \in [d]$, by the symmetry of the initialization,
	with probability $\frac{1}{2}$ over the initialization of $\wvec{i}_0$,
	we get that $\inner{\wvec{i}_0,\elm_j} \le 0$. Since all $\ww_i$'s are initialized
	independently, we get that:
	\[
	P(\elm_j \in K_0) = P(\cap_{i\in[k]} \inner{\wvec{i}_0,\elm_j} \le 0) = 
	\prod_{i \in [k]} P(\inner{\wvec{i}_0, \elm_j} \le 0) = \frac{1}{2^k}
	\]
	Now, assuming $k \le \log_2 (\frac{d}{-\ln(\delta)})$, from the independence of the initialization of $\wvec{i}_0$'s coordinates
	we get
	\begin{align*}
	P(\cap_{j \in [d]}
	\elm_j \notin K_0)
	&= \prod_{j \in [d]} P(\elm_j \notin K_0) \\
	&= (1-\frac{1}{2^k})^d \le e^{-\frac{d}{2^k}} \le \delta
	\end{align*}
	Therefore, with probability at least $1-\delta$, there exists $j \in [k]$
	for which $\elm_j \in K_0$. By Lemma \ref{lem:orth_minimum}, this implies that
	for all $t \in \naturals$ we will get $\elm_j \in K_t$, and therefore
	$N_{\mw_t} (e_j) \le 0$. Since $\elm_j$ is labeled $1$,
	this implies that $\loss{S} > 0$.
	By the separability of the data, and by the convergence of the SGD algorithm,
	this implies that the algorithm converges to a stationary point that is
	not a global minimum. Note that convergence to a saddle point is
	possible only if we define $\sigma^{\prime}(0) = 0$, and for all
	$i \in [k]$ we have at the time of convergence $\inner{\wvec{i}_t,\elm_j} = 0$.
	This can only happen if $\inner{\wvec{i}_0, \elm_j} = \eta N$ for some
	$N \in \naturals$, which has probability zero over the initialization of
	$\wvec{i}_t$. Therefore, the convergence is almost surely to a non-global minimum point.\\
	%TODO I think we should properly prove the convergence of the SGD
	
	On the other hand, assuming $k \ge \log_2(\frac{d}{\delta})$,
	using the union bound we get:
	\begin{align*}
	P(\cup_{j \in [d]}
	\elm_j \in K_0)
	& \le \sum_{j \in [d]}
	P(\elm_j \in K_0) \\
	& = \frac{d}{2^k} \le \delta
	\end{align*}
	So with probability at least $1 - \delta$, we get $K_0 = \emptyset$
	and by Lemma \ref{lem:orth_minimum} this means $K_t = \emptyset$ for all
	$t \in \naturals$.
	Now, if $e_j \notin K_t$ for all $t \in \naturals$,
	then there exists $i \in [k]$ such that 
	$\inner{\wvec{i}_t, e_j} > 0$ for all $t \in \naturals$.
	If after performing $T$ update iterations we have
	updated $N > \max\{\frac{C}{\eta},\frac{1}{\eta}\}$ times on $e_j$, then clearly:
	\begin{align*}
	& \inner{\wvec{i}_t, \elm_j} = \inner{\wvec{i}_0,\elm_j} + \sum_{t=0}^T \eta \inner{ \elm_j, \elm_j}
	\ge \inner{\wvec{i}_0, \elm_j} + N \eta > 1 \\
	& \forall_{ i \in [k]} ~s.t~ \inner{\uvec{i}_0, \elm_j} > 0, ~
	\inner{\uvec{i}_t, \elm_j} =
	\inner{\uvec{i}_0,\elm_j} - \sum_{t=0}^T \eta \inner{ \elm_j, \elm_j} \le C - N\eta \le 0
	\end{align*}
	and therefore $N_{\mw_t}(\elm_j) > 1$, which implies that
	$L_{\{(\elm_j,1)\}}(W_t) = 0$. From this, we can conclude
	that for each $j \in [d]$, we perform at most
	$\lceil\max\{\frac{C}{\eta},\frac{1}{\eta}\}\rceil$ update iterations on $\elm_j$ before reaching zero loss, and therefore we can perform at most
	$\lceil \max\{\frac{dC}{\eta},\frac{d}{\eta}\}\rceil$
	update iterations until convergence. Since we show that we never get stuck
	with zero gradient on an example with loss greater than zero,
	this means we converge to a global optimum after at most
	$\lceil \max\{\frac{dC}{\eta},\frac{d}{\eta}\}\rceil$ iterations.

\end{document}